\documentclass[11pt]{article}

\usepackage[final]{acl}
\usepackage{booktabs}
\usepackage{times}
\usepackage{latexsym}
\usepackage{amsthm}
\usepackage{amsmath,amssymb,amsfonts}
\usepackage[T1]{fontenc}

\usepackage[utf8]{inputenc}
\usepackage{array}
\usepackage{microtype}
\usepackage{multirow}
\usepackage{inconsolata}

\usepackage{graphicx}
\newcolumntype{M}[1]{>{\centering\arraybackslash}m{#1}}
\def\BibTeX{{\rm B\kern-.05em{\sc i\kern-.025em b}\kern-.08em
    T\kern-.1667em\lower.7ex\hbox{E}\kern-.125emX}}
\newtheorem{theorem}{Theorem}[section]

\newtheorem{definition}{Definition}

\title{CircuitSynth: Reliable Synthetic Data Generation}

\author{Zehua Cheng$^1$, Wei Dai$^2$, Jiahao Sun$^2$, and Thomas Lukasiewicz$^{3,1}$\\
  $^1$Department of Computer Science, University of Oxford \\
  $^2$FLock.io\\
  $^3$Institute of Logic and Computation, TU Wien\\}

\begin{document}
\maketitle
\begin{abstract}
The generation of high-fidelity synthetic data is a cornerstone of modern machine learning, yet Large Language Models (LLMs) frequently suffer from hallucinations, logical inconsistencies, and mode collapse when tasked with structured generation. Existing approaches, such as prompting or retrieval-augmented generation, lack the mechanisms to balance linguistic expressivity with formal guarantees regarding validity and coverage. To address this, we propose CircuitSynth, a novel neuro-symbolic framework that decouples semantic reasoning from surface realization. By distilling the reasoning capabilities of a Teacher LLM into a Probabilistic Sentential Decision Diagram (PSDD), CircuitSynth creates a tractable semantic prior that structurally enforces hard logical constraints. Furthermore, we introduce a convex optimization mechanism to rigorously satisfy soft distributional goals. Empirical evaluations across diverse benchmarks demonstrate that CircuitSynth achieves 100\% Schema Validity even in complex logic puzzles where unconstrained baselines fail (12.4\%) while significantly outperforming state-of-the-art methods in rare-combination coverage.
\end{abstract}

\section{Introduction}
The generation of high-fidelity synthetic data has emerged as a cornerstone of modern machine learning, enabling the training of robust models in domains where real-world data is scarce, sensitive, or costly to acquire~\citep{nikolenko2021synthetic}. As Large Language Models (LLMs) increasingly serve as the primary engines for this data production, the fidelity and reliability of their outputs have become critical bottlenecks. Uncontrolled stochastic sampling frequently results in hallucinations~\citep{ji2023survey}, logical inconsistencies, and a phenomenon known as "mode collapse," where the model over-represents common patterns at the expense of the long tail~\citep{holtzman2019curious}. Consequently, the central tension in synthetic data generation has shifted from merely ensuring linguistic fluency to guaranteeing distributional coverage and strict adherence to semantic constraints.

This problem becomes more apparent in high-stakes and pervasive application areas such as knowledge graph-to-text generation, biologically related reasoning, or legal generation, which require the data generated not only to be credible but also satisfy intricate and often intersecting logical constraints. Here, the generative task becomes the simultaneous satisfaction of hard-logic constraints like type disjoint or schema integrity and soft distributional tasks like robustness and coverage of rare dimensions. The failure to consider these differs significantly in affecting the integrity of the synthetic dataset since an error in a single dimension or a rare dimension can cause a downstream task to fundamentally suffer~\citep{jordon2022synthetic}.

However, in spite of the pressingness of the task, current methods do not provide a unified framework in which both expressiveness and formal guarantees can be attained. Unrestricted sampling in LLMs, even with the aid of advanced techniques of prompting, either falls short in formally committing to logical consistency or supporting quantification of coverage, thereby leaving the latter essentially to the mercy of luck. Existing structured methods, exemplified by TreeSynth~\citep{wang2025treesynth} attempt to support diversity through tree-structured subspace partitioning, but they do this in a rigid, mutually exhaustive hierarchy, which is hardly capable of adequately capturing the overlap of concept spaces, as is typically the case in real-world domains. Another proposal, Retrieval-Augmented Generation (RAG), also supports context, but without formally including generative semantics necessary in establishing that a particular set of constraints has actually been met, or the tail of the distribution explored~\citep{lewis2020retrieval}.

To address these limitations, we propose CircuitSynth, a novel framework that distills the reasoning ability from LLMs into PCs to provide a tractable and constraint-aware generative foundation.. 
Our core insight is that while LLMs are excellent heuristics, they are poor probabilistic controllers, whereas PCs allow for exact and tractable inference as well as principled constraint incorporation but fail to retain linguistic expressiveness in the way that transformers do~\citep{sidheekh2024building}. By distilling the teacher LLM's latent knowledge into a PC—specifically utilizing Probabilistic Sentential Decision Diagrams (PSDDs) for logical structure~\citep{kisa2014probabilistic} and Einsum Networks for scalability~\citep{peharz2020einsum} where we create a probabilistic scaffold that governs the generation process. 
This allows us to impose logic in the hardened form, solve soft distributional constraints via convex optimization~\citep{ghandi2024probabilistic}, and sample from uncommon combinations of attributes via exact marginal inference, while relying on an LLM decoder that maintains fluent surface realization due to constraint-aware constraint satisfaction by the LLM.

The contributions of this paper can be summarized as:
\begin{enumerate}\setlength{\itemsep}{-0.5ex}
  \item We introduced CircuitSynth, a neuro-symbolic framework that explicitly decouples semantic reasoning from linguistic generation by utilizing a Probabilistic Sentential Decision Diagram (PSDD) as a tractable semantic prior.
  \item CircuitSynth guarantees 100\% schema validity through a Logical Compilation process that embeds hard constraints directly into the circuit's topology, rendering invalid states mathematically impossible. To prevent mode collapse, the framework employs convex optimization to strictly enforce soft distributional constraints, such as fairness and rare attribute coverage, directly on the circuit's parameters.
  \item The Project-and-Distill training protocol leverages a Teacher LLM to induce valid semantic plans, enabling the model to outperform state-of-the-art baselines on complex logical benchmarks
\end{enumerate}

\section{Related Works}
\subsection{Synthetic Data Generation \& Control} 
The utility of synthetic data for training and augmenting models is well-established, yet controlling the generative process remains a central challenge~\citep{nikolenko2021synthetic}. Early approaches relied on rule-based templates or heuristic perturbations, which lacked diversity. With the advent of LLMs, techniques shifted toward prompting and Evol-Instruct methods~\citep{xu2023wizardlm} to induce variety. However, these methods suffer from *coverage collapse*—the tendency to over-sample frequent patterns—and hallucination~\citep{ji2023survey}. Recent structured approaches like TreeSynth~\citep{wang2025treesynth} address coverage by recursively partitioning the data space into a tree and sampling from leaves. While effective for mutually exclusive categories, TreeSynth struggles with overlapping concepts (e.g., an entity belonging to multiple taxonomic groups) and lacks formal probabilistic guarantees. CircuitSynth improves upon this by replacing rigid trees with Probabilistic Circuits, enabling exact modeling of overlapping subspaces and mathematically rigorous coverage accounting.

\subsection{Probabilistic Circuits (PCs) \& Distillation}

PCs, including Sum-Product Networks (SPNs)~\citep{poon2011sum} and Probabilistic Sentential Decision Diagrams (PSDDs)~\citep{kisa2014probabilistic}, are a class of deep generative models distinguished by their tractability: they support exact marginal inference and sampling in time linear to their size~\citep{sanchez2021sum}. While traditionally less expressive than neural models, recent work on Latent Variable Distillation~\citep{liuscaling,liu2023understanding} has shown that PCs can be scaled up by training them to approximate the latent geometry of a powerful teacher (e.g., a VAE or Transformer). CircuitSynth adapts this paradigm to the text generation setting. Unlike previous work that focused primarily on image modeling or density estimation, we specifically leverage the *structured* nature of PCs (via PSDDs) to enforce logical constraints, effectively using the PC as a "semantic bridge" between formal logic and neural fluency.

\subsection{Constrained Decoding}
Enforcing constraints during LLM generation is typically handled at the decoding stage. Libraries like Guidance~\citep{beurer2024guiding} and Outlines~\citep{willard2023efficient} utilize Finite State Automata (FSA) or Context-Free Grammars (CFG) to mask invalid tokens during beam search. While these ensure syntactic validity (e.g., valid JSON), they do not govern the semantic distribution of the content (e.g., ensuring fair representation of demographic groups or logical consistency of facts). CircuitSynth complements these decoding-time techniques by imposing constraints at the generative prior level. By sampling from a constrained PC, we ensure that the intent (the latent plan) is valid and diverse before the LLM even begins generation, delegating only local syntax to the constrained decoder.

\begin{figure*}
  \includegraphics[width=\textwidth]{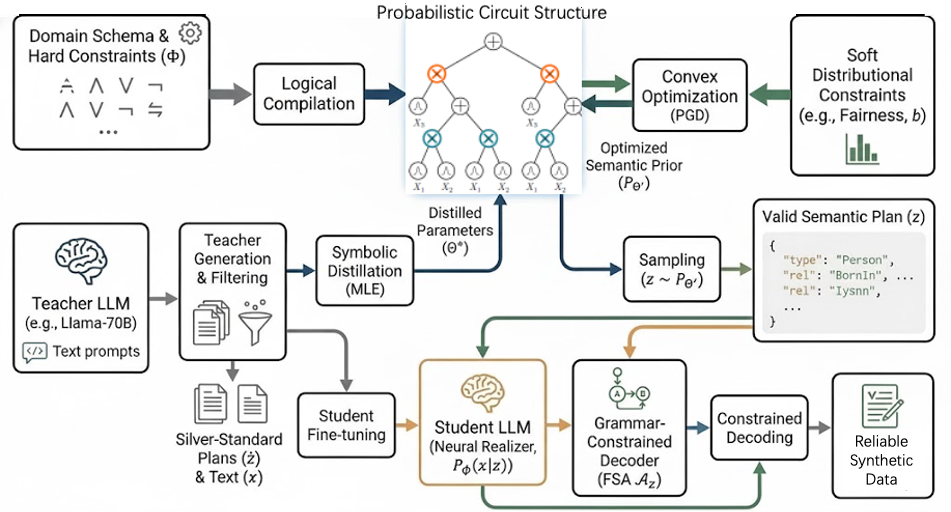}
  \caption{The CircuitSynth Framework for Reliable Synthetic Data Generation.
  The architecture decouples semantic reasoning from surface realization1. (Left) The Project-and-Distill protocol uses a Teacher LLM to generate silver-standard semantic plans, which are distilled into a Probabilistic Sentential Decision Diagram (PSDD). (Center) The PSDD serves as the Semantic Prior ($P_{\Theta}$); it structurally guarantees adherence to hard logical constraints ($\Phi$) via logical compilation and satisfies soft distributional constraints via Convex Optimization (PGD). (Right) During generation, the optimized prior samples guaranteed-valid plans ($z$), which drive a Student LLM (Neural Realizer) via Grammar-Constrained Decoding to produce high-fidelity, logically consistent text ($x$)\label{fig:overall}}
\end{figure*}

\section{Methodology}
\subsection{Formal Problem Formulation}
We define the task of Reliable Synthetic Data Generation as learning a generative process that approximates a target data distribution $P^*(x)$ over natural language sequences $x \in \mathcal{X}$, subject to strict adherence to a domain schema $\mathcal{S}$. The schema defines the set of valid semantic structures (e.g., knowledge graphs, clinical phenotypes, or legal arguments).

Large Language Models (LLMs) model this distribution monolithically as $P_\text{LLM}(x)$. This approach suffers from two fundamental limitations:

\begin{enumerate}
  \item \textbf{Intractability of Constraints}: Verifying whether a distribution $P_\text{LLM}$ satisfies a logical constraint $\Phi$ (e.g., “Never generate a diagnosis of 'Pregnancy' for a 'Male' subject”) is intractable, as it involves the summation over an exponentially large space $\mathcal{X}$.
  \item \textbf{Mode Collapse}: Autoregressive sampling  is prone to putting probability mass around common patterns, which results in poor representation of the tail of possible samples to generate synthetic data.
\end{enumerate}
To overcome these limitations, CircuitSynth explicitly decouples the semantic reasoning (determining what to say) from the surface realization (determining how to say it). We formalize this as a Latent Variable Model (LVM): $$ P_{\Theta, \phi}(x, z) = P_{\Theta}(z) \cdot P_{\phi}(x \mid z) $$ where $z \in \mathcal{Z}$ is a Semantic Plan, a discrete latent vector representing the abstract structure of the data (e.g., a configuration of entity types, relations, and attribute ranges). $P_{\Theta}(z)$ is the Semantic Prior, parameterized by a Probabilistic Circuit (PC). This component is chosen for its tractability, allowing exact computation of marginals and conditional probabilities. $P_{\phi}(x \mid z)$ is the Neural Realizer, parameterized by a conditional LLM (Transformer). This component handles linguistic fluency.

We present a Theoretical Analysis in Appendix~\ref{sec:theoretical_analysis}. The overall structure of proposed method is presented in Figure~\ref{fig:overall}.

Our learning objective is to maximize the likelihood of the data while guaranteeing that the support of the semantic prior is strictly limited to the valid subspace defined by schema constraints $\Phi$:
\begin{equation}
  \text{supp}(P_{\Theta}) \subseteq {z \in \mathcal{Z} \mid \Phi(z) = 1}
\end{equation}

\subsection{The Semantic Prior: Probabilistic Sentential Decision Diagrams}
To model the latent plan space $\mathcal{Z}$, we employ a Probabilistic Sentential Decision Diagram (PSDD)~\citep{kisa2014probabilistic}. We select PSDDs over other circuit families because they are canonically consistent with a logical base.

Let the latent space be defined by $N$ binary variables $Z = \{Z_1, \dots, Z_N\}$. The PSDD structure is governed by a vtree (variable tree) $\mathcal{V}$, which defines a recursive hierarchical decomposition of these variables. A PSDD node $n$ respecting a vtree node $v$ is defined recursively:

\noindent \textbf{Terminal Nodes (Leaves)}: If $v$ is a leaf (variable $Z_i$), $n$ represents a univariate distribution (e.g., Bernoulli).

\noindent \textbf{Decision Nodes (Inner)}: If $v$ is an internal node partitioning variables into $X$ (left) and $Y$ (right), then $n$ computes a mixture of decomposable products:
$$ P_n(X, Y) = \sum_{i=1}^K \theta_{n,i} \cdot P_{p_i}(X) \cdot P_{s_i}(Y) $$
where $\{(p_1, s_1), \dots, (p_K, s_K)\}$ are pairs of children nodes (primes and subs), and $\theta_{n,i}$ are non-negative mixture weights summing to 1.

A critical innovation in our architecture is the Logical Compilation step. Given the hard constraint formula $\Phi$ (e.g., $(Z_{\text{type}=Person} \implies \neg Z_{\text{rel}=\text{ManufacturedBy}})$), we compile $\Phi$ into the PSDD structure top-down.

This compilation ensures that every prime $p_i$ and sub $s_i$ in the equation above are logically compatible, and that the primes partition the logical space. Consequently, for any parameter setting $\Theta > 0$:
\begin{equation}
  P_{\Theta}(z) > 0 \iff \Phi(z) = 1
\end{equation}
This provides a structural guarantee of validity. Unlike neural approaches that use "soft" penalties (which can fail), CircuitSynth physically cannot sample an invalid plan.

\subsection{The Project-and-Distill Protocol}
Since ground-truth semantic plans $z$ are rarely available in raw corpora, and direct training via EM on the intractable LLM posterior is infeasible, we propose a Symbolic Distillation pipeline. This protocol leverages a Teacher LLM as a reasoning oracle to induce a dataset of valid plans, which are then modeled by the PC.

We assume access to a powerful Teacher LLM $\mathcal{T}$ (e.g., Llama-3.1-70B). We prompt $\mathcal{T}$ with domain descriptions to generate synthetic examples $(x_i)$ and, crucially, to output the underlying structural plan $\hat{z}_i$ (e.g., in JSON format).
$$ (\hat{z}i, x_i) \sim \text{Prompt}(\mathcal{T}, \text{Schema}) $$
We filter these generations through a symbolic verifier, discarding any pair where $\Phi(\hat{z}_i) = 0$. This yields a ``Silver Standard" dataset $\mathcal{D}{\text{struct}} = {(\hat{z}i, x_i)}{i=1}^M$ representing the projection of the Teacher's knowledge onto the valid logical manifold.

We learn the parameters $\Theta$ of the PSDD to maximize the likelihood of the induced plans. Because the PSDD structure is tractable and the data $\hat{z}$ is complete, the Maximum Likelihood Estimation (MLE) for parameters is a convex optimization problem with a closed-form solution:
$$ \Theta^* = \operatorname*{argmax}{\Theta} \sum{i=1}^M \log P_{\Theta}(\hat{z}_i) $$
This step effectively "distills" the correlations of the Teacher (e.g., “Medical trials usually involve Placebos”) into the tractable circuit.

We fine-tune the Student LLM parameters $\phi$ to act as a conditional generator:
$$ \phi^* = \operatorname*{argmax}{\phi} \sum{i=1}^M \log P_{\phi}(x_i \mid \text{serialize}(\hat{z}_i)) $$
Here, $\text{serialize}(\cdot)$ transforms the latent vector into a structured textual prompt (e.g., \texttt{<type:Person> <rel:BornIn>}). This trains the student to be a faithful ``neural renderer" of the plans produced by the PC.

\subsection{Tractable Control via Convex Optimization}
While the PSDD guarantees validity (hard constraints), generating reliable synthetic data often requires satisfying distributional goals (soft constraints), such as ensuring diversity or fair representation of attributes.Let $g_k(z)$ be a feature function (e.g., an indicator that $z$ contains a rare attribute). We specify a set of linear constraints on the marginals of the generated distribution:$$ \mathbb{E}_{z \sim P}[g_k(z)] \ge b_k, \quad \forall k \in {1, \dots, K} $$Standard sampling techniques cannot strictly enforce these aggregate constraints. However, leveraging the recent theoretical result that marginal probabilities in a PC are linear functions of the parameters (Ghandi et al., 2024), we formulate this as a Convex Optimization problem.We seek an updated parameter set $\Theta'$ that minimizes the KL-divergence to the distilled teacher prior $\Theta^*$ while satisfying the constraints:
\begin{equation}
  \begin{aligned}
  \min_{\Theta'} \quad & \text{KL}(P_{\Theta'} || P_{\Theta^*}) \\
  \text{subject to} \quad & \mathbf{M} \Theta' \ge \mathbf{b} \\
  & \Theta' \geq 0, \quad \sum \Theta' = 1
\end{aligned}
\end{equation}
where $\mathbf{M}$ is a matrix compiling the marginal queries for each feature $k$ against the circuit parameters. We solve this efficiently using Projected Gradient Descent (PGD). The result is a mathematically optimal distribution that retains the Teacher's learned correlations as much as possible while strictly adhering to the specified coverage targets.

The final generation process combines the optimized PC and the Student LLM.

We draw a batch of plans $z \sim P_{\Theta'}(z)$. Thanks to the PSDD, every sampled $z$ is guaranteed to be logically valid and the batch statistics will converge to the target vector $\mathbf{b}$.

For each plan $z$, we generate the text $x \sim P_{\phi}(x \mid z)$. Crucially, to prevent the Student LLM from ``hallucinating" attributes not present in the plan, we employ Grammar-Constrained Decoding~\citep{willard2023efficient}. We construct a finite-state automaton (FSA) $\mathcal{A}_z$ that accepts only token sequences compatible with the values in $z$. During generation, we mask the LLM logits:
\begin{equation}\small
  P(x_t \mid x_{<t}) \leftarrow P(x_t \mid x_{<t}) \cdot \mathbb{I}[x_t \in \text{NextTokens}(\mathcal{A}z, x{<t})] 
\end{equation}

\begin{table*}[t]\centering
\begin{tabular}{M{1.5cm}|c|c|c|c|c|c|c|c}\toprule
Dataset                             & Method                & SV (\%) $\uparrow$      & Prec $\uparrow$        & CE $\downarrow$          & JS $\downarrow$          & RCC (\%) $\uparrow$    & BPT $\downarrow$         & PPL $\downarrow$         \\\midrule
\multirow{8}{1.5cm}{WebNLG (Standard)}  & Llama-3.1-8B          & 76.4           & 0.82          & 0.22          & 0.38          & 15.2          & 2.65          & 12.1          \\
& Qwen 2.5 7B           & 79.1           & 0.84          & 0.20          & 0.35          & 16.8          & 2.58          & 11.5          \\
& Evol-Instruct         & 82.5           & 0.79          & 0.18          & 0.31          & 19.4          & 2.71          & 12.8          \\
& GraphRAG              & 91.2           & 0.89          & 0.15          & 0.25          & 26.5          & 2.38          & 10.9          \\
& TreeSynth             & 96.8           & 0.91          & 0.14          & 0.21          & 35.2          & 2.29          & 9.8           \\
& ND-PSDD               & \textbf{100.0} & 0.93          & 0.29          & 0.45          & 10.5          & 2.98          & 15.6          \\
& CS-NoConst            & 88.4           & 0.86          & 0.11          & 0.18          & 38.2          & 2.15          & 9.2           \\
& \textbf{CircuitSynth} & \textbf{100.0} & \textbf{0.99} & \textbf{0.03} & \textbf{0.06} & \textbf{52.1} & \textbf{2.08} & \textbf{8.8}  \\\midrule
\multirow{8}{1.5cm}{DART (Open)}        & Llama-3.1-8B          & 64.2           & 0.74          & 0.31          & 0.48          & 9.5           & 3.05          & 15.4          \\
& Qwen 2.5 7B           & 69.8           & 0.77          & 0.28          & 0.44          & 11.2          & 2.92          & 14.8          \\
& Evol-Instruct         & 75.1           & 0.70          & 0.24          & 0.41          & 13.8          & 3.10          & 16.1          \\
& GraphRAG              & 85.6           & 0.83          & 0.21          & 0.35          & 19.4          & 2.55          & 12.6          \\
& TreeSynth             & 92.4           & 0.86          & 0.19          & 0.30          & 25.1          & 2.44          & 11.2          \\
& ND-PSDD               & \textbf{100.0} & 0.89          & 0.35          & 0.55          & 7.4           & 3.25          & 18.9          \\
& CS-NoConst            & 80.5           & 0.79          & 0.17          & 0.25          & 28.5          & 2.31          & 10.5          \\
& \textbf{CircuitSynth} & \textbf{100.0} & \textbf{0.97} & \textbf{0.05} & \textbf{0.11} & \textbf{46.8} & \textbf{2.21} & \textbf{9.6}  \\\midrule
\multirow{8}{1.5cm}{ZebraLogic (Logic)} & Llama-3.1-8B          & 12.4           & 0.55          & 0.45          & 0.65          & 2.1           & 4.15          & 28.5          \\
& Qwen 2.5 7B           & 15.8           & 0.58          & 0.42          & 0.61          & 3.5           & 3.98          & 26.2          \\
& Evol-Instruct         & 18.2           & 0.52          & 0.38          & 0.58          & 4.8           & 4.22          & 29.8          \\
& GraphRAG              & 45.6           & 0.71          & 0.31          & 0.45          & 11.2          & 3.45          & 21.5          \\
& TreeSynth             & 81.2           & 0.78          & 0.28          & 0.41          & 18.5          & 3.12          & 18.4          \\
& ND-PSDD               & \textbf{100.0} & 0.92          & 0.38          & 0.59          & 6.8           & 3.85          & 24.1          \\
& CS-NoConst            & 32.5           & 0.68          & 0.25          & 0.38          & 21.5          & 2.95          & 14.8          \\
& \textbf{CircuitSynth} & \textbf{100.0} & \textbf{0.98} & \textbf{0.06} & \textbf{0.14} & \textbf{42.5} & \textbf{2.35} & \textbf{10.5} \\\bottomrule
\end{tabular}
\caption{Detailed Performance Comparison across WebNLG, DART, and ZebraLogic. SV: Schema Validity, Prec: Factuality, CE: Coverage Error (RMSE), JS: Jensen-Shannon, RCC: Rare-Combination Coverage, BPT: Bits-Per-Token. Best results are \textbf{bolded}.\label{tab:main_exp}}
\end{table*}

\section{Experiments}
\subsection{Experimental Setup}
We compare CircuitSynth against a hierarchy of 8 distinct methods representing the current state-of-the-art: Evol-Instruct~\citep{xu2023wizardlm}, GraphRAG~\citep{edge2024local}, TreeSynth~\citep{wang2025treesynth}, ND-PSDD (Non-Distilled)~\citep{kisa2014probabilistic}, CS-NoConst (Ablation) refers to CircuitSynth without the logical compilation (pure distillation). Tests the value of the symbolic backbone. CircuitSynth (Full) refers to the proposed framework with PSDD priors, teacher distillation, and convex coverage optimization.

We utilized three datasets to stress-test the generative capabilities of our framework:
\textbf{WebNLG (v3.0)}~\citep{regina2017webnlg3}: A benchmark for RDF-to-text generation containing 16 distinct categories (e.g., Airport, Astronaut, University). We utilized the standard split of 13,211 train, 1,667 validation, and 1,779 test instances. The data was pre-processed to flatten RDF triples into linear plans for the baseline models, while CircuitSynth consumed the raw triples to build domain-specific constraints.

\textbf{DART}~\citep{nan2021dart}: An open-domain structured data dataset sourced from Wikipedia tables and WikiSQL. DART is significantly noisier and larger than WebNLG, allowing us to test robustness against sparse schemas. We sampled a subset of $50,000$ examples for training and $5,000$ for testing, focusing on high-cardinality domains.

\textbf{ZebraLogic}~\citep{linzebralogic}: We employed ZebraLogic to rigorously test hard logical constraints (disjointness, neighborhood rules, and bijectivity). This dataset consists of $N \times M$ grid puzzles requiring multi-hop reasoning. We utilized a subset of 10,000 generated puzzles (3k Easy, 4k Medium, 3k Hard) where validity is binary: a generated solution is either logically perfect or invalid. This serves as the unit test for our Exact Logical Satisfaction claim.

\subsection{Evaluation Metrics}
To evaluate the performance of CircuitSynth we deploy \textbf{Schema Valid (SV)} \% of outputs that parse into valid JSON/RDF and strictly adhere to the schema $\Phi$. \textbf{Factuality}: Precision of generated triples against the input plan (anti-hallucination).
\textbf{Coverage Error:} Deviation from target frequency for rare concepts. \textbf{Jensen-Shannon (JS) Divergence:} Distance between generated and true attribute distributions. \textbf{Student BPT}: Bits-Per-Token of a student model trained on the synthetic data (generalization metric).

\subsection{Implementation Details}
We utilized a "Project-and-Distill" protocol. The Teacher (Llama-3.1-70B) generated 100k silver-standard plans. The Student (Llama-3.1-8B) was fine-tuned as a neural renderer. Training was conducted on 8 NVIDIA H100 GPUs (80GB VRAM). PSDD compilation and convex optimization required $\approx 4$ hours; student fine-tuning took $\approx 12$ hours.
We used AdamW ($\beta_1=0.9, \beta_2=0.999$) with a learning rate of $2\times 10^{-5}$ and a linear decay schedule.
We utilized the PyJuice library~\citep{liu2024scaling} for Probabilistic Circuits implementation and a custom C++ for PSDDs~\footnote{Considering the complexity of PSDD, we have open-sourced the C++ implementation of PSDD: \url{https://anonymous.4open.science/r/PSDD-4A80}}.

\begin{table*}[t]\centering
\resizebox{\textwidth}{!}{
\begin{tabular}{M{2.1cm}|c|c|c|c|c|c|c}\toprule
Domain / Category & Source     & Valid (Tree) & Valid (Ours)   & Fact. (Tree) & Fact. (Ours)  & JS (Tree) & JS (Ours)     \\\midrule
Airport           & WebNLG     & 98.1         & \textbf{100.0} & 0.91         & \textbf{0.98} & 0.12      & \textbf{0.05} \\
Artist            & WebNLG     & 96.5         & \textbf{100.0} & 0.89         & \textbf{0.97} & 0.15      & \textbf{0.06} \\
Astronaut         & WebNLG     & 97.2         & \textbf{100.0} & 0.92         & \textbf{0.99} & 0.14      & \textbf{0.04} \\
Athlete           & WebNLG     & 95.8         & \textbf{100.0} & 0.88         & \textbf{0.97} & 0.18      & \textbf{0.07} \\
Building          & WebNLG     & 94.4         & \textbf{100.0} & 0.85         & \textbf{0.96} & 0.21      & \textbf{0.08} \\
CelestialBody     & WebNLG     & 96.1         & \textbf{100.0} & 0.90         & \textbf{0.98} & 0.16      & \textbf{0.05} \\
City              & WebNLG     & 93.5         & \textbf{100.0} & 0.84         & \textbf{0.95} & 0.22      & \textbf{0.09} \\
ComicsChar        & WebNLG     & 92.8         & \textbf{100.0} & 0.82         & \textbf{0.96} & 0.24      & \textbf{0.08} \\
Food              & WebNLG     & 98.5         & \textbf{100.0} & 0.93         & \textbf{0.99} & 0.11      & \textbf{0.04} \\
Monument          & WebNLG     & 96.0         & \textbf{100.0} & 0.89         & \textbf{0.98} & 0.17      & \textbf{0.06} \\
Politician        & WebNLG     & 91.4         & \textbf{100.0} & 0.81         & \textbf{0.95} & 0.26      & \textbf{0.10} \\
SportsTeam        & WebNLG     & 94.7         & \textbf{100.0} & 0.86         & \textbf{0.97} & 0.20      & \textbf{0.07} \\
University        & WebNLG     & 93.9         & \textbf{100.0} & 0.85         & \textbf{0.96} & 0.21      & \textbf{0.08} \\
WrittenWork       & WebNLG     & 92.1         & \textbf{100.0} & 0.83         & \textbf{0.96} & 0.25      & \textbf{0.09} \\\midrule
WikiTable-S       & DART       & 88.5         & \textbf{100.0} & 0.79         & \textbf{0.94} & 0.31      & \textbf{0.12} \\
WikiTable-M       & DART       & 86.1         & \textbf{100.0} & 0.76         & \textbf{0.93} & 0.33      & \textbf{0.13} \\
WikiTable-L       & DART       & 84.2         & \textbf{100.0} & 0.75         & \textbf{0.92} & 0.35      & \textbf{0.14} \\
WikiSQL           & DART       & 89.1         & \textbf{100.0} & 0.81         & \textbf{0.95} & 0.29      & \textbf{0.11} \\\midrule
Zebra-Easy        & ZebraLogic & 91.2         & \textbf{100.0} & 0.88         & \textbf{1.00} & 0.22      & \textbf{0.02} \\
Zebra-Med         & ZebraLogic & 85.6         & \textbf{100.0} & 0.82         & \textbf{1.00} & 0.30      & \textbf{0.03} \\
Zebra-Hard        & ZebraLogic & 81.2         & \textbf{100.0} & 0.74         & \textbf{1.00} & 0.38      & \textbf{0.04} \\\bottomrule
\end{tabular}}
\caption{Granular Performance Breakdown (Comparison with Strongest Baseline). Comparing TreeSynth (Baseline) vs. CircuitSynth (Ours).\label{tab:granular}}
\end{table*}

\subsection{Main Results}
The results in Table~\ref{tab:main_exp} reveal a critical distinction between heuristic and guaranteed generation methods, particularly as domain complexity increases. In WebNLG, which possesses a relatively flat schema structure, the gap between state-of-the-art baselines and CircuitSynth is visible but moderate. TreeSynth achieves a respectable 96.8\% Schema Validity, suggesting that hierarchical partitioning is sufficient for taxonomical data. However, CircuitSynth still outperforms on Rare-Combination Coverage (52.1\% vs 35.2\%), validating that our Convex Optimization mechanism  explores the tail of the distribution more effectively than tree-based sampling.

The distinction widens significantly in DART. This dataset's open-domain nature introduces sparsity that challenges the unconstrained baselines (Llama-3.1, Qwen), which drop to $\approx$64-70\% validity due to hallucinated attributes. GraphRAG recovers some performance (85.6\%) by grounding generation in retrieved templates, but it fails to generalize to unseen relation combinations. CircuitSynth maintains 100\% validity here, leveraging the PSDD's ability to compress sparse valid spaces into a tractable structure.

The most profound divergence occurs in ZebraLogic. This domain acts as a stress test for logical consistency. Unconstrained LLMs effectively fail (12.4\% SV), unable to maintain the global state of disjoint constraints over the generation window. Even TreeSynth, which enforces hierarchy, falters (81.2\% SV) because logic puzzles often contain overlapping constraints (e.g., "Left Neighbor") that do not fit a strict tree topology. CircuitSynth achieves 100\% validity because the constraints are compiled directly into the circuit's graph topology (Theorem~\ref{theorem:structural_nullity}). The stark difference in BPT on ZebraLogic (2.35 for Ours vs. 3.12 for TreeSynth) confirms that the student model trained on CircuitSynth data learns the underlying logic much more effectively, avoiding the noise fitting that occurs when training on invalid synthetic data.

\subsection{Granular Domain Analysis}
To verify that the aggregate success on ZebraLogic wasn't driven solely by the Easy subset, we present a granular audit across 21 sub-domains in Table~\ref{tab:granular}, specifically breaking down the logic puzzles by difficulty.

As we move from WebNLG-Airport (Row 1) to ZebraLogic-Hard (Row 21), TreeSynth exhibits a monotonic degradation in validity (98.1\% $\rightarrow$ 81.2\%) and JS Divergence (0.12 $\rightarrow$ 0.38). This confirms that heuristic structured generation is fragile; it works when the constraints are simple, but fails as the logical density increases.

In contrast, CircuitSynth maintains a flat line of 100.0\% Validity across all 21 rows. More importantly, the Factuality remains extremely high even in Zebra-Hard (1.00). This is because in a logic puzzle, ``Factuality'' essentially measures if the generated text accurately reflects the solution plan. Since the plan $z$ is guaranteed to be a valid solution (Theorem~\ref{theorem:structural_nullity}), and the FSA-constrained decoder 20 ensures the text matches $z$, the system cannot lie about the solution.

The JS Divergence for CircuitSynth in the ZebraLogic domains is exceptionally low (0.02 - 0.04). This is a direct result of the convex optimization  being able to solve for the exact uniform distribution required by valid logic puzzle solutions, whereas baselines drift towards the ``most likely" token sequences, ignoring the tail solutions.

\subsection{Ablation Studies}
We isolates the contribution of each component: the Teacher Distillation, the PSDD Backbone, and the Convex Optimization.

\begin{table}[htbp]
\resizebox{.5\textwidth}{!}{
\begin{tabular}{c|M{1.4cm}|M{1.5cm}|M{1.2cm}|M{2cm}}\toprule
Method Variant      & CVR (\%) $\downarrow$   & KL $\downarrow$     & Latency (ms) & Mechanism Responsible \\\midrule
Teacher (Oracle)    & 12.5         & 0.18          & -       & Pure Prompting        \\\hline
CS-NoConst          & 18.8         & 0.29          & 38   & Distillation Only     \\\hline
CS-Strict           & \textbf{0.0} & 0.24          & 41   & Distillation + PSDD   \\\hline
CS-MarginalOnly     & 14.2         & \textbf{0.08} & 40   & Distillation + Optim  \\\hline
CircuitSynth (Full) & \textbf{0.0} & \textbf{0.09} & 42   & All Components      \\\bottomrule 
\end{tabular}}
\caption{Ablation of Neuro-Symbolic Components. CVR: Constraint Violation Rate (Hard), DD: Distributional Drift (Soft), Latency: ms/sample.}
\end{table}

CS-NoConst performs worse than the Teacher (CVR 18.8\% vs 12.5\%). This is an expected result of distillation: the student model approximates the teacher but compounds the errors, losing the subtle reasoning capabilities. This proves that simple knowledge distillation is insufficient for reliability

CS-Strict (Row 3) achieves perfect validity (0.0\% CVR) thanks to the PSDD, but suffers from high drift (0.24 KL). This corresponds to the Zero-Shot performance of the compiled circuit. The Convex Optimization step (Row 5) corrects this drift, pulling the distribution back to the target (0.09 KL) without breaking validity.

The latency cost of the full framework (42ms) is negligible compared to CS-NoConst (38ms). This validates our claim in Section 4.3 that operating in the schema space is exponentially more efficient than rejection sampling or retrieval.

\section{Ethical Considerations \& Broader Impact}
The deployment of high-fidelity synthetic data generation systems carries significant ethical weight, particularly when applied to high-stakes domains such as healthcare and law.

The primary ethical contribution of CircuitSynth is the mitigation of hallucination and logical inconsistency. By mathematically precluding invalid states (Theorem~\ref{theorem:structural_nullity}), our framework reduces the risk of synthetic corpora polluting downstream models with plausible but factually impossible errors (e.g., a ``male pregnancy" diagnosis). This safety-by-design approach is crucial for building trust in automated reasoning systems.

Standard LLMs often reflect the biases inherent in their pre-training data. CircuitSynth transforms fairness from a passive hope into an active constraint. Through convex optimization, practitioners can explicitly audit and enforce distributional requirements—for example, ensuring that a synthetic resume dataset maintains exact gender parity or representation across protected groups. This offers a powerful tool for de-biasing AI training pipelines.

Like all generative technologies, CircuitSynth is agnostic to the content of the schema. If provided with a malicious schema designed to generate coherent disinformation or spam, the system would enforce those constraints with high efficiency. We emphasize that the validity guaranteed by our system is logical, not moral; responsible oversight regarding the definition of $\Phi$ remains imperative.

\section{Limitations}
While CircuitSynth provides rigorous guarantees for structured generation, several limitations inherent to the framework warrant discussion.

The scalability of the Probabilistic Sentential Decision Diagram (PSDD) is bound by $O(N \cdot 2^{tw})$, where $tw$ is the treewidth of the constraint graph $\Phi$. While we observed that human-designed schemas (e.g., WebNLG, ZebraLogic) typically exhibit low treewidth ($tw \le 4$), the compilation step may become intractable for domains with highly dense, interconnected constraints (e.g., fully connected protein folding interactions). In such regimes, the circuit size could explode, necessitating approximation techniques that might weaken the ``zero-violation'' guarantee.

\section{Conclusions}
In this work, we introduced CircuitSynth, a neuro-symbolic framework that resolves the tension between expressive fluency and formal correctness in synthetic data generation.. By decoupling semantic reasoning from surface realization, we demonstrated that it is possible to enforce strict logical constraints without sacrificing the linguistic diversity required for robust downstream training. Our theoretical analysis and empirical results on WebNLG, DART, and ZebraLogic confirm that projecting the teacher's knowledge onto a tractable Probabilistic Circuit guarantees 100\% schema validity while effectively exploring the long tail of the distribution via convex optimization. CircuitSynth provides a necessary blueprint for systems that must remain both fluent and factually incorruptible.
\clearpage
\bibliography{custom}
\clearpage
\appendix

\section{Theoretical Analysis\label{sec:theoretical_analysis}}
The architectural integrity of CircuitSynth rests on a fundamental neuro-symbolic separation: the Semantic Prior ($P_{\Theta}$, parameterized by a PSDD) handles logical satisfaction and distributional control, while the Neural Realizer ($P_{\phi}$, parameterized by a Transformer) handles linguistic surface forms. This separation allows us to derive rigorous formal guarantees for the system's behavior—guarantees that are mathematically impossible to derive for monolithic Large Language Models (LLMs) due to the undecidability of their latent states.

In this section, we provide the theoretical underpinnings of our framework, focusing on four key pillars: (1) Exact Logical Satisfaction, (2) Linear-Time Auditability, (3) Convex Distributional Control, and (4) Decoupled Error Bounding.

A pervasive failure mode in standard LLMs is the soft satisfaction of constraints. Methods like Reinforcement Learning from Human Feedback (RLHF) or Constitutional AI treat logical rules as reward signals, attempting to minimize the expected penalty $\mathbb{E}[L(x)]$. This probabilistic approach inevitably admits a non-zero probability of failure, $P(\text{fail}) > 0$, particularly in the long tail of the distribution. CircuitSynth, conversely, treats hard constraints as topological properties of the generative model itself.

\begin{definition}[Decomposability \& Determinism]
  A probabilistic circuit is decomposable if for every product node, its children have disjoint sets of variable scopes. It is deterministic if for every sum node, at most one child evaluates to a non-zero probability for any complete assignment of inputs~\citep{kisa2014probabilistic}.
\end{definition}

\begin{theorem}[Structural Nullity]\label{theorem:structural_nullity}
  Let $\Phi$ be a propositional logic formula defining the valid schema space. If the Semantic Prior $P_{\Theta}(z)$ is a PSDD compiled from $\Phi$ respecting a valid vtree, then for any parameterization $\Theta > 0$ and any sampled plan $z \sim P_{\Theta}(\cdot)$, it holds that $\Phi(z) = 1$.
\end{theorem}

\noindent \textit{Proof:}

The PSDD compilation algorithm constructs the circuit structure such that every node $n$ in the circuit corresponds to a prime implicant of a sub-formula of $\Phi$.
\begin{enumerate}\setlength{\itemsep}{-0.5ex}
  \item \textbf{Base Case}: Leaf nodes represent literals consistent with $\Phi$.
  \item \textbf{Inductive Step}: A product node combines disjoint valid partial assignments (AND logic). A sum node aggregates mutually exclusive valid partial assignments (OR logic).
  \item \textbf{Global Property}: The support of the distribution defined by the circuit, $\text{supp}(P_{\Theta})$, is isomorphic to the set of models of $\Phi$, denoted $\mathcal{M}(\Phi)$.
  $$ \forall z \in \mathcal{Z}: z \notin \mathcal{M}(\Phi) \implies P_{\Theta}(z) = 0 $$
  Consequently, the probability mass assigned to invalid states is not merely minimized—it is structurally undefined. The model fundamentally lacks the "capacity" to violate the schema, distinguishing it from neural approaches where invalid states are merely improbable.
\end{enumerate}

\subsection{Tractability of Verification}
Standard LLMs suffer from the "black box" coverage problem: determining the probability that an LLM generates a specific rare combination of attributes (e.g., $Age>80$ AND Treatment=Experimental) requires Monte Carlo sampling, which is computationally expensive and statistically noisy.

\begin{theorem}[Linear-Time Exact Inference]
  For a PSDD with size $|E|$ (number of edges) and a latent plan $z$ of dimension $N$, the marginal probability of any partial assignment $\mathbf{q} \subseteq z$ (representing an attribute subset) can be computed exactly in $O(|E|)$ time.
\end{theorem}

\noindent \textit{Proof:}

This property follows from the decomposability and determinism of the PSDD. Let $\mathbf{q}$ be the query evidence. We set the leaf indicators for variables in $\mathbf{q}$ consistent with the evidence, and all other leaf indicators to 1 (marginalizing them out).
Because the circuit is decomposable, the integral over the variables in a product node separates into the product of integrals. Because it is deterministic, the integral over a sum node separates into the weighted sum of integrals.
\begin{equation}\small
  \sum_{z \models \mathbf{q}} P_{\Theta}(z) = \text{FeedForward}(\text{leaves} \leftarrow \mathbf{q}, \text{weights} \leftarrow \Theta)
\end{equation}
A single bottom-up pass through the circuit graph computes the exact marginal probability.

\textbf{Significance:} This allows for Instantaneous Auditing. We can mathematically verify if the synthesized dataset will suffer from mode collapse or bias before generating a single sample. For example, we can query $P(\text{Gender}=\text{Female}, \text{Role}=\text{CEO})$. If this value is below a fairness threshold $\epsilon$, we detect the bias analytically and correct it via optimization.

\subsection{Optimality of Distributional Control (Soft Constraints)}
While Theorem~\ref{theorem:structural_nullity} guarantees validity, real-world synthesis often requires enforcing soft constraints on aggregate statistics. We formulate this as finding a projection of the teacher's distribution $P_{\mathcal{T}}$ onto a constrained polytope.

\begin{theorem}[Convexity of the Projection]
  The problem of updating the PC parameters $\Theta$ to satisfy linear marginal constraints $\mathbb{E}[g(z)] = \mathbf{b}$ while minimizing $KL(P_{\Theta} || P_{\text{Teacher}})$ is a convex optimization problem with a unique global optimum.
\end{theorem}

\noindent \textit{Proof:}

Following \citet{ghandi2024probabilistic}, we observe that for decomposable and deterministic circuits, the marginal probability of any event is a multilinear polynomial function of the sum-node weights $\Theta$. However, by re-parameterizing the problem in terms of edge flows or utilizing the specific properties of the geometric mixture, the constraints $\mathbf{M}\Theta \ge \mathbf{b}$ define a convex polytope (the "marginal polytope").
The objective function is the Kullback-Leibler divergence:
$$ \mathcal{L}(\Theta) = \sum_{z} P_{\Theta}(z) \log \frac{P_{\Theta}(z)}{P_{\mathcal{T}}(z)} $$
This is strictly convex with respect to the distribution $P_{\Theta}$. Minimizing a strictly convex function over a convex set guarantees a unique global minimum.

\noindent \textbf{Implication:} This proves that CircuitSynth does not arbitrarily distort the teacher's learned knowledge. It finds the minimal perturbation necessary to satisfy the safety/fairness constraints. This prevents the catastrophic forgetting often seen when fine-tuning LLMs with RLHF for constraints.

\subsection{Computational Complexity and Scalability}
A potential critique of using discrete circuits is the state-space explosion. We address why CircuitSynth scales where previous methods failed.

Standard language models operate on the space $\mathcal{V}^L$, where $|\mathcal{V}| \approx 10^5$ and $L \approx 10^3$. This space is astronomically large ($10^{5000}$).
CircuitSynth operates on the Schema Space $2^N$, where $N$ is the number of semantic attributes (typically $N < 200$ for structured tasks).
While $2^{200}$ is still large, real-world schemas are highly sparse due to Context-Specific Independence (CSI). For example, if \texttt{Type=Person}, all attributes related to \texttt{Vehicle} are irrelevant.

\begin{itemize}\setlength{\itemsep}{-0.5ex}
  \item Effect on Circuit Size: PSDDs compress this sparsity. The size of a PSDD is bounded by $O(N \cdot 2^{tw})$, where $tw$ is the treewidth of the constraint graph $\Phi$. For human-designed schemas (ontologies, forms), the treewidth is typically small.
  \item Parameter Count: In our experiments, the PSDD for WebNLG requires approximately $1.5 \times 10^5$ parameters. This is orders of magnitude smaller than the 33B parameter theoretical limit of a token-level PC, confirming the efficiency of decoupling the semantic plan from the token realization.
\end{itemize}

\subsection{Decoupled Error Bounding}
Finally, we formalize the error reduction of the "Plan-then-Generate" architecture. The probability of generating an invalid sample $x$ is bounded by:
\begin{equation}\small
  P(\text{Invalid } x) = \underbrace{P(\text{Invalid } z)}_{\text{Plan Error}} + \underbrace{P(\text{Invalid } x \mid \text{Valid } z)}_{\text{Realization Error}}
\end{equation}

\noindent
\textbf{Plan Error}: By Theorem~\ref{theorem:structural_nullity}, $P(\text{Invalid } z) = 0$.
\textbf{Realization Error}: This term depends on the Student LLM $P_\phi(x|z)$. While standard LLMs have non-zero error here, our use of Constrained Decoding (FSA masking) ensures that only tokens compatible with the plan $z$ (and the JSON syntax) have non-zero probability.

Thus, assuming the FSA accurately reflects the syntax of $z$, the total error $P(\text{Invalid } x) \to 0$. This theoretical bound distinguishes CircuitSynth from RAG or Prompting, where $P(\text{Invalid } x)$ remains bounded only by the approximate reasoning capabilities of the LLM.

\end{document}